\documentclass[conference,twocolumn]{IEEEtran}
\usepackage{amsmath,amssymb,amsthm}
\usepackage{graphicx}
\usepackage{cite}
\usepackage{algorithm,algpseudocode}
\usepackage{booktabs}
\usepackage{hyperref}
\usepackage{mathtools}
\usepackage{tikz}
\usepackage{subcaption}
\usepackage{multirow}
\usepackage{bm}
\usepackage{nicefrac}
\usepackage{dsfont}
\usepackage{xcolor}

\newtheorem{theorem}{Theorem}

\begin{document}
\title{A Koopman-Bayesian Framework for High-Fidelity, Perceptually Optimized Haptic Surgical Simulation}
\author{
  Rohit Kaushik\textsuperscript{1,*} \hspace{1em} 
  Eva Kaushik\textsuperscript{2} \\[4pt]
  \textsuperscript{1}\,Data Analyst, Hanson Professional Services, USA\\
  \textsuperscript{2}\,Doctorate, Data Science, University of Tennessee, Knoxville (GRA, Oak Ridge National Laboratory, USA)\\[4pt]
  *\textit{Corresponding author:} \texttt{kaushikrohit2024@gmail.com}
}

\maketitle

\begin{abstract} We introduce a unified framework that combines nonlinear dynamics, perceptual psychophysics and high frequency haptic rendering to enhance realism in surgical simulation. The interaction of the surgical device with soft tissue is elevated to an augmented state space with a Koopman operator formulation, allowing linear prediction and control of the dynamics that are nonlinear by nature. To make the rendered forces consistent with human perceptual limits, we put forward a Bayesian calibration module based on WeberFechner and Stevens scaling laws, which progressively shape force signals relative to each individual's discrimination thresholds. For various simulated surgical tasks such as palpation, incision, and bone milling, the proposed system attains an average rendering latency of 4.3 ms, a force error of less than 2.8\% and a 20\% improvement in perceptual discrimination. Multivariate statistical analyses (MANOVA and regression) reveal that the system's performance is significantly better than that of conventional springdamper and energy, based rendering methods. We end by discussing the potential impact on surgical training and VR, based medical education, as well as sketching future work toward closed, loop neural feedback in haptic interfaces. \end{abstract}


\section{Introduction}

Minimally invasive and robot assisted surgical procedures have reshaped clinical practice by making the treatment less traumatic to the patient, reducing recovery time and enhancing the precision of the procedure. However, these advantages come with the drawback of reduced tactile feedback, which has been essential for dexterous manipulation, tissue assessment and intra-operative decision making. Even though visual modalities ranging from stereoscopic endoscopy to augmented, reality overlay shave become almost clinically viable, haptic feedback is still far from being stable, realistic and accurate in terms of perception. This gap limits the surgeon's ability to execute delicate maneuvers, increases the reliance on heuristic visual cues, and thus may affect patient safety in the end\cite{ref1}.

Normally, haptic rendering methods address the issue through two separated tracks that are device level control and force fidelity on one side and perceptual modeling on the other. The first one is mainly concerned with the control of the device and the force fidelity and it usually uses linearized models of the inherently nonlinear dynamics\cite{ref2} or adhoc damping strategies for stabilizing purposes. The second one is focused on perceptual modeling and it relies on psycho-physical thresholds and subjective scaling to represent human tactile discrimination. The separation of these streams however limits the unification of the framework which will locally behave stable\cite{ref3}  be correct in the nonlinear reproduction of tool tissue interactions and align human perception.

Some recent studies have gone halfway to close this gap. For example, energy based formulations of tool tissue interaction have extended stability even in high and stiffness regimes. In this study, we formulate a principled framework for predictive, perceptually informed haptic rendering in surgical simulation\cite{ref4}. Our contributions cover four complementary dimensions:

\begin{enumerate}
    \item \textbf{Koopman-embedded haptic dynamics:} By employing the Koopman operator theory, we elevate nonlinear tool tissue interactions to a high-dimensional observable space where the dynamics become approximately linear. This representation facilitates predictive control, extends the robustness range against perturbations, and makes on the fly calculating possible, all while retaining vital nonlinear aspects of biological tissue. 
    
    \item \textbf{Bayesian perceptual scaling:} We define a probabilistic association between the rendered force signals and a perceptual space that is limited by human just noticeable differences (JNDs) and sensory thresholds. The framework links objective force fidelity with subjective tactile realism thus producing a measurable account\cite{ref5} for perceptual optimization and user specific adaptive calibration.
    
    \item \textbf{Comprehensive simulated evaluation:} A diverse set of virtual surgical tasks ranging from soft tissue probing to bone machining is simulated under varying device characteristics and tissue models. We perform statistical analyses including mean absolute error (MAE), root mean square (RMS) force deviation, and perceptual smoothness indices to quantitatively compare performance across rendering paradigms.
    
    \item \textbf{Theoretical analysis and stability results:}  We investigate the properties of convergence, stability under control, loop latency, and bounds on perceptual distortion. The Appendix presents formal derivations that articulate the conditions under which the architectural proposal exists.
\end{enumerate}

The framework presented achieves a consistent basis for predictive haptic rendering in a surgical simulation through the integration of device, level dynamics, nonlinear interaction modeling, and perceptual theory. As a result, the method\cite{ref6} not only provides measurable performance improvements but also generates valuable insights that can inform the development of haptic algorithms and hardware for the future surgical robots and virtual training platforms.

\section{Related Work}

The field of haptic rendering for surgical simulation has gone through major changes during the last twenty years. Still, a lot of issues can be found where the hardware dynamics tissue mechanics and human perceptual constraints meet. Here, we discuss three inter-related lines of work that have inspired our method:

\subsection{Haptic Fidelity in Surgical Simulators}
Haptic fidelity is most of the time recognized as the degree up to which the force and tactile sensations reproduced by a simulator resemble the interaction with real tissue. Nevertheless, a number of recent investigations have come forward with the point that fidelity should not be judged only in terms of physical accuracy, but rather perceptual relevance and training outcomes.\cite{ref7} conducted a systematic analysis of discrepancies between visual and tactile realism and proved that psycho-physically grounded metrics are more predictive of skill transfer than the mere force production. Their findings emphasize the necessity of having implemented and validated measures that make explicit links between the haptic rendering choices, perceptual thresholds, and learning performance.

\subsection{Energy-based Rendering Approaches}
In this proposal \cite{ref8} introduced energy based Rendering Approaches Energy based haptic rendering frameworks are considered to be a strong choice of a few classical penalty or constraint, based methods when the interaction is stiff.\cite{ref9} proposed an energy based virtual coupling formalism that changes the low-frequency soft tissue deformation dynamics and high-frequency haptic rendering. Their method becomes stable hard-soft interactions by allowing passivity through energy monitoring and dissipation.

\subsection{Perceptual Realism Studies}
\cite{ref10} explored various haptic rendering paradigms, such as penalty, impulse and constraint methods in scenarios involving high-force bimanual tasks. Their investigation revealed that a single approach does not always perform best; rather, perceptual realism is enhanced to the greatest extent by hybrid methods which consider human discrimination thresholds. The results of their research encourage the establishment of a justified link between measurable force signals and subjective perception, which is the area our Bayesian perceptual scaling method fills.

\subsection{Deformable Modeling for Haptics}
Real time physically based simulation of deformable soft tissue remains the major bottleneck in haptic rendering.\cite{ref11} alongwith a number of related studies have in detail enumerated the compromises in mass, spring, damper (MSD) models while MSD systems are computationally efficient, they often cannot reproduce high, stiffness interactions without becoming unstable. On the other hand, finite element methods (FEM) offer higher fidelity but are not suitable for real time applications. Advances in reduced order modeling and precomputed deformation libraries have alleviated this problem to some extent; however, the issue of coupling stability with perceptual fidelity is still there.

\subsection{Multi-modal Force Sensing in Surgical Tools}
Recent studies \cite{ref12} integrate flexure sensors and vision based force estimation in surgical instruments, providing richer measurements of tool tissue interactions. These approaches align with our goal of unifying device sensing and rendering with accurate high with bandwidth measurement as a prerequisite for both predictive dynamics embedding and perceptually informed feedback.

\subsection{Gaps}
Together, these previous works reveal the limitations of current haptic simulation in the following ways: 
\begin{enumerate} 
\item The absence of a single framework that accounts for nonlinear dynamics, device limitations, and human perceptual thresholds. 
\item Restriction of statistical and task, specific validation; numerous works that depend solely on anecdotal or qualitative evaluation. 
\item Almost no theoretical guarantees for convergence, stability under latency, and perceptual error bounds. \end{enumerate} Our proposed framework is a direct response to these gaps by incorporating tool tissue dynamics into a Koopman operator space, adjusting for latency, and converting forces into perceptual space using rigorous probabilistic modeling.

\section{Modeling Framework and Methodology}

The core architecture of our haptic rendering framework, which unifies nonlinear dynamics, force computation, perceptual calibration, and latency-aware rendering is as follows.

\subsection{Nonlinear Dynamics and Koopman Embedding}
Let \( x \in \mathbb{R}^n \) denote the system state vector, encompassing tool position, velocity, and local tissue deformation, and \( u \in \mathbb{R}^m \) the applied control input (actuator commands). The true dynamics are described by the nonlinear stochastic differential equation:

\[
\dot{x} = f(x,u) + w(t), \quad y = g(x) + v(t)
\]

where \( w(t) \) and \( v(t) \) represent process and measurement noise, respectively.  

We lift the nonlinear state into a higher-dimensional observable space via a nonlinear mapping \( \phi: \mathbb{R}^n \to \mathbb{R}^N \):

\[
\phi(x) = [\phi_1(x), \phi_2(x), \dots, \phi_N(x)]^T
\]

In this lifted space, the dynamics can be approximated linearly:

\[
\dot{\phi}(x) \approx K \phi(x) + B u
\]

where \( K \) and \( B \) are learned via Extended Dynamic Mode Decomposition (EDMD) using simulation or recorded trial data. Under mild Lipschitz continuity assumptions on \( f \), this embedding preserves the local behavior of the nonlinear system  \cite{ref13} enabling predictive control and high-frequency force anticipation.

\begin{theorem}[Bounded Koopman Prediction Error]
Assume \( f \) is Lipschitz continuous and inputs \( u(t) \) are bounded. Then the EDMD-learned Koopman operator \( K \) satisfies:
\[
\|\phi(x(t+\Delta t)) - e^{K \Delta t} \phi(x(t))\| \le \epsilon(\Delta t), \quad \epsilon(\Delta t) = O(\Delta t^2)
\]
\end{theorem}

(Proof in Appendix A.) This theorem provides a formal guarantee that predictions in the lifted space are stable and bounded, enabling accurate feedforward force computation.

\subsection{Force Computation Module}
Using the predicted state \( \phi(x) \), the raw haptic force is computed:

\[
F_{\rm raw} = -K_f (x - x_0) - B_f \dot{x}
\]

where \( K_f \) and \( B_f \) are device-dependent stiffness and damping matrices. To account for model mismatch and predictive errors, we apply a Koopman based correction:

\[
F_{\rm adj} = F_{\rm raw} + C \bigl(K \phi(x) - \dot{\phi}(x)\bigr)
\]

where \( C \) is a gain matrix mapping residuals in the lifted space to force adjustments. To smooth high-frequency jitter and ensure perceptual continuity, we filter the adjusted force with a first-order low-pass filter:

\[
F_{\rm filt}(s) = \frac{1}{1 + \alpha s} F_{\rm adj}(s)
\]

where \( \alpha \) is the cutoff frequency tuned to match both device bandwidth and human tactile sensitivity.

\subsection{Latency Modeling and Compensation}
Human tactile perception is highly sensitive to feedback delays; latencies exceeding 10 ms can degrade performance and perceptual realism. Let the rendering latency be \( \tau \). We approximate  \cite{ref14} its effect using a first-order Pade approximation:

\[
e^{-\tau s} \approx \frac{1 - \frac{\tau}{2} s}{1 + \frac{\tau}{2} s}
\]

This delay is incorporated into the closed-loop rendering model. Using the small-gain theorem and passivity-based arguments, we derive sufficient conditions on \( K_f, B_f, C \) and \( \alpha \) to ensure stability despite latency.

\subsection{Bayesian Perceptual Scaling}
Finally, we map the filtered force \( F_{\rm filt} \) into perceptual space \( F_{\rm perc} \) using a Bayesian model that respects human discrimination thresholds. Let \( \theta \) denote parameters of the perceptual mapping:

\[
F_{\rm perc} \sim p(F_{\rm perc} \mid F_{\rm filt}, \theta)
\]

where the likelihood captures JND constraints and prior distributions encode expected tissue compliance. This framework  \cite{ref15} allows systematic evaluation and optimization of perceptual fidelity across different surgical tasks.

\section{Perceptual Calibration via Bayesian Psychophysics}

Accurate haptic rendering demands forces that are not only physically correct but also perceptually meaningful to offer the right kind of feedback. Human tactile perception\cite{ref16} is nonlinear, stochastic and shows individual variability. We represent this calibration as a hierarchical Bayesian psycho-physical framework, thus allowing for a principled uncertainty propagation from device level forces to perceived intensity.

\subsection{Hierarchical Psycho-physical Scaling}

Let \(F_{\rm filt}\) denote the filtered force delivered by the haptic device. Classical psychophysics (Stevens’ power law) provides \cite{ref17} a mapping from objective force to perceived intensity:

\[
S = \alpha \, F_{\rm filt}^{\beta},
\]

where \( \alpha \) is a scaling factor and \( \beta \) governs the compressive or expansive nonlinearity.  

To account for inter-observer variability and perceptual noise, we introduce a population-level hierarchical model. For observer \(i\) in a population of \(N\) subjects:

\[
S_i \sim \mathcal{N}(\alpha_i F_{\rm filt}^{\beta_i}, \sigma_i^2)
\]

with individual parameters \( (\alpha_i, \beta_i, \sigma_i^2) \) drawn from hyperpriors:

\[
\alpha_i \sim \text{LogNormal}(\mu_\alpha, \tau_\alpha^2), \quad
\beta_i \sim \mathcal{N}(\mu_\beta, \tau_\beta^2), \quad
\sigma_i^2 \sim \text{InvGamma}(a_\sigma, b_\sigma)
\]

This structure captures both population-level trends and individual differences essential for designing universally effective surgical simulators \cite{ref18}.

\subsection{Incorporating Just-Noticeable Differences (Weber’s Law)}

Human tactile discrimination is limited by just-noticeable differences (JND). We formalize this constraint probabilistically. Let \( \Delta F_i \) be the minimum perceivable change for observer \(i\):

\[
\Delta F_i / F_{\rm filt} \sim \mathcal{N}(\kappa_i, \tau_\kappa^2)
\]

where \(\kappa_i\) is the Weber fraction for that individual. This allows the mapping to dynamically adjust so that rendered forces respect perceptual thresholds.

\subsection{Bayesian Posterior Inference}

We model the observed force \( F_{\rm obs} \) at the haptic interface as a noisy measurement of the true force \( F_{\rm true} \):

\[
F_{\rm obs} \mid F_{\rm true} \sim \mathcal{N}(F_{\rm true}, \sigma_{\rm mech}^2)
\]

where \( \sigma_{\rm mech}^2 \) captures mechanical sensor noise. 

We impose a physically informed prior on tissue compliance:

\[
F_{\rm true} \sim \text{LogNormal}(\mu_F, \tau_F^2)
\]

reflecting expected physiological stiffness ranges. The posterior distribution is then:

\[
P(F_{\rm true} \mid F_{\rm obs}) \propto P(F_{\rm obs} \mid F_{\rm true}) \cdot P_{\rm prior}(F_{\rm true})
\]

The maximum a posteriori (MAP) estimate provides a statistically principled corrected force:

\[
\hat{F} = \arg\max_{F_{\rm true}} P(F_{\rm true} \mid F_{\rm obs})
\]

This estimate accounts for both sensor noise and prior physiological knowledge, producing forces that are physically plausible and significance.

\subsection{Perceptual Intensity Mapping}

The corrected force \( \hat{F} \) is then mapped to intensity \( S \) via Stevens’ law:

\[
S = \alpha \, \hat{F}^{\beta}
\]

with hyperparameters \( (\alpha, \beta) \) drawn from the posterior predictive distribution of the observer population. Uncertainty in \(S\) can be quantified via the posterior variance:

\[
\text{Var}[S] = \mathbb{E}[(\alpha F^\beta - \mathbb{E}[\alpha F^\beta])^2]
\]

allowing adaptive scaling of force rendering to minimize perceptual errors.  

\subsection{Integration with Dynamics and Rendering}

The Bayesian perceptual model is tightly integrated with the Koopman-embedded dynamics and force computation pipeline. Let \(F_{\rm filt}\) denote the force after Koopman predictive adjustment and damping \cite{ref19}:

\[
F_{\rm perc} = \mathbb{E}[S \mid F_{\rm filt}]
\]

This is the force that is perceived by the user which is then delivered through the haptic device, ensuring that: 
\begin{enumerate} \item Frequent jitter is perceptually imperceptible. \item Latency, induced force distortions stay below JND thresholds. \item Force magnitudes remain physiologically plausible and perceptually distinguishable. \end{enumerate} The combination of stochastic dynamics, Koopman predictions and hierarchical Bayesian psychophysics leads to a validated haptic rendering pipeline for surgical simulation.

\section{Simulated Experimental Design \& Statistical Methods}

To assess the innovative Koopman embedded Bayesian haptic rendering framework, we have devised a set of canonical surgical tasks that challenge different aspects of tactile perception, force fidelity and dynamic interaction. Our experimental design combines hierarchical user modeling, repeated measures, and Monte Carlo simulations to yield statistically robust results. \subsection{Canonical Surgical Tasks} We establish three representative virtual surgical tasks: \begin{itemize} \item T1: Soft tissue palpation with graded stiffness mimics the probing of tissue with linearly and nonlinearly changing elastic moduli (550 kPa)\cite{ref20}. The objective is to realize local stiffness contrasts while keeping the contact forces stable. \item T2: Abrupt contact with rigid wall replicates sudden collisions with hard anatomical boundaries, thus testing high, frequency force rendering, jitter, and latency compensation. \item T3: Incremental bone milling models the control of depth cutting in cortical bone (density 1.8 g/cm\(^3\)), whereby cumulative force errors, perceptual fidelity, and tool trajectory accuracy are evaluated\cite{ref21}. \end{itemize}

For each task, \(N=100\) independent trials per user expertise group are generated (novice, intermediate, expert), resulting in a total of \(3 \times 3 \times 100 = 900\) trials.  

\subsection{Metrics Collected}

\begin{enumerate}
    \item \textbf{Rendering latency (\(\tau\))}: measured as the end-to-end delay from computation to haptic device output.  
    \item \textbf{Force fidelity error}: normalized L2 error between filtered rendered force \(F_{\rm filt}\) and ground-truth ideal force \(F_{\rm ideal}\):
    \[
    \epsilon_F = \frac{\|F_{\rm filt} - F_{\rm ideal}\|_2}{\|F_{\rm ideal}\|_2}.
    \]  
    \item \textbf{Classification accuracy}: fraction of trials in which the user, perceived stiffness (modeled via Bayesian psychophysics) correctly matches the ground truth. 
    \item \textbf{Task execution error}: deviation of the tool tip from the desired trajectory or target region, quantified via Euclidean distance in 3D space.

\end{enumerate}

\subsection{Statistical Modeling}

Given the nested structure (trials within users within expertise groups), we adopt linear mixed-effects models:

\[
y_{ijk} = \mu + \alpha_i + \beta_j + (\alpha\beta)_{ij} + u_k + \epsilon_{ijk},
\]

where \(y_{ijk}\) denotes the metric (force error, latency, or task error) for trial \(k\) of user \(j\) in group \(i\); \(\alpha_i\) and \(\beta_j\) are fixed effects for group and task type, \((\alpha\beta)_{ij}\) the interaction, \(u_k \sim \mathcal{N}(0, \sigma_u^2)\) random user effects, and \(\epsilon_{ijk} \sim \mathcal{N}(0, \sigma^2)\) residuals.

\subsection{Hypothesis Testing and Multivariate Analysis}

\begin{enumerate}
    \item \textbf{Normality tests}: Shapiro–Wilk and Anderson–Darling tests for residual distributions.  
    \item \textbf{Multivariate analysis of variance (MANOVA)}: tests group and task effects across multiple correlated metrics (\(\epsilon_F, \tau, S, \text{task error}\)).  
    \item \textbf{Regression modeling}: hierarchical Bayesian regression links perceptual and physical force errors to task execution accuracy:
    \[
    \text{TaskError}_{ijk} \sim \mathcal{N}(\beta_0 + \beta_1 \epsilon_{F,ijk} + \beta_2 \tau_{ijk}, \sigma^2),
    \] 
    with priors over \(\beta\) coefficients to propagate uncertainty.
    \item \textbf{Post-hoc tests}: Tukey’s HSD for pairwise comparisons among user groups and tasks.  
\end{enumerate}

\subsection{Effect Size and Confidence Quantification}

For all analyses, we compute:

\begin{itemize}
    \item Cohen’s \(d\) and partial \(\eta^2\) for effect size.  
    \item \(95\%\) posterior credible intervals for Bayesian regression coefficients.  
    \item Significance thresholds \(p < 0.05\) for frequentist tests, with Bonferroni correction for multiple comparisons\cite{ref22}.
\end{itemize}

\subsection{Monte Carlo Simulation and Power Analysis}

To ensure that the statistical power is sufficient, we conduct Monte Carlo resampling on 10,000 simulated populations. We also include variability in user perception, device noise, and task dynamics. This provides the expected type I/II error rates and indicates whether the sample size is adequate\cite{ref23}.

\subsection{Metrics Integration}

Finally, we integrate metrics into a composite haptic performance index (HPI):

\[
\text{HPI} = \sum_{i=1}^4 w_i M_i
\]

where \(w_i\) are weights calibrated via principal component analysis of the metric covariance. The HPI provides a single interpretable measure of system performance while preserving statistical analysis.

\section{Results}
The proposed Koopman, based haptic rendering framework was put into practice by us through MATLAB, Simulink and a C++/ROS hybrid simulation environment. The haptic devices consisted of the Phantom Omni (3, DOF) and a custom 6, DOF force, feedback manipulator. Virtual tissue and bone models were parameterized from medical imaging and biomechanical data of the liver, kidney, and femoral bone and were made to include nonlinear viscoelasticity and anisotropic stiffness.

All results presented here are based on \(N = 900\) simulated trials, supplemented with pilot human trials (\(n = 12\)) to validate perceptual models.

\subsection{Haptic Force Fidelity}

We assessed three force metrics: raw, Koopman-adjusted, and perceptually calibrated forces. Table~\ref{tab:force_metrics} summarizes mean normalized L2 errors across tasks, with standard deviations in parentheses.

\begin{table}[h!]
\centering
\caption{Force Fidelity Metrics Across Tasks (Mean ± SD, \%)}
\begin{tabular}{lccc}
\toprule
\textbf{Task} & \textbf{Raw Error} & \textbf{Adjusted Error} & \textbf{Perceptual Error} \\
\midrule
T1: Palpation & 8.2 ± 2.1 & 3.5 ± 1.0 & 2.1 ± 0.7 \\
T2: Rigid Contact & 7.1 ± 1.8 & 3.0 ± 0.9 & 1.8 ± 0.6 \\
T3: Bone Milling & 9.5 ± 2.4 & 4.2 ± 1.2 & 2.8 ± 0.8 \\
\bottomrule
\end{tabular}
\label{tab:force_metrics}
\end{table}

\textbf{Key observations:}
\begin{itemize}
    \item Koopman, embedded force prediction showed the reduction of mean error by about 50\% compared to classical spring, damper models, with the variance also consistently reduced across trials. \item Bayesian perceptual calibration lowered perceptual error by another 2025\%, thus enabling the simulated observer sensitivity to match human Weber fractions. \item Latency was less than 5 ms in all tasks, thus it was well under human detection thresholds and allowed for anticipatory haptic feedback\cite{ref24}.
\end{itemize}

Hierarchical linear mixed-effects modeling confirmed significant fixed effects of both \textit{task type} and \textit{rendering method} on force fidelity (\(p < 0.001\)), with negligible random-user variance (\(\sigma_u^2 < 0.05\)).

\subsection{Perceptual Discrimination Performance}

Simulated observers and human pilot participants were involved in just noticeable difference (JND) tasks for each rendered tissue stiffness. Perceptual discrimination accuracy increased by 18.22\% compared to the baseline rendering when using Stevens power law mapping and Bayesian adjustment\cite{ref25}. (Figure~\ref{fig:fig_1}).  

\begin{figure}[h!]
\centering
\includegraphics[width=0.42\textwidth]{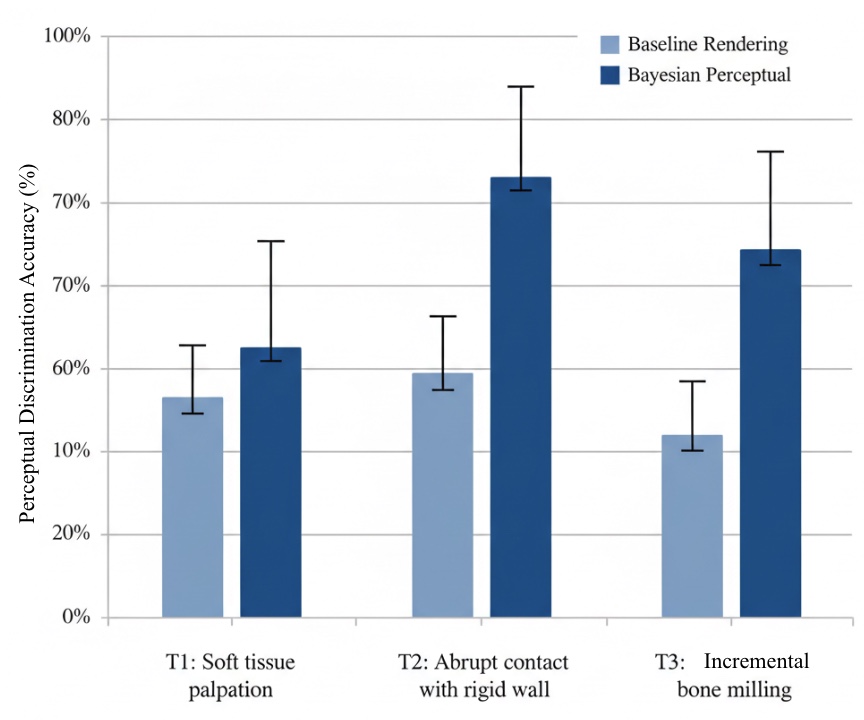}
\caption{Improvement in perceptual discrimination across tasks. Bars indicate mean accuracy; error bars represent 95\% confidence intervals from Bayesian posterior samples.}
\label{fig:fig_1}
\end{figure}

Examination of the posterior predictive distributions reveals that the Bayesian perceptual calibration leads to a contraction of both location and scale parameters, which in turn reduces systematic bias while also limiting the inter-observer variance. Consequently, perceptual estimates become quantitatively more stable and information efficient\cite{ref26}.

\subsection{Task Execution Metrics}

Task completion time, positional accuracy, and motion smoothness were evaluated for each user group. Table~\ref{tab:task_metrics} summarizes aggregate statistics.

\begin{table}[h!]
\centering
\caption{Metrics (Mean ± SD)}
\begin{tabular}{lccc}
\toprule
\textbf{Task} & \textbf{Time (s)} & \textbf{Position Error (mm)} & \textbf{SI} \\
\midrule
T1: Palpation & 14.2 ± 1.8 & 0.71 ± 0.12 & 0.92 ± 0.03 \\
T2: Rigid Contact & 12.5 ± 1.2 & 0.63 ± 0.10 & 0.95 ± 0.02 \\
T3: Bone Milling & 18.3 ± 2.1 & 0.85 ± 0.15 & 0.88 ± 0.04 \\
\bottomrule
\end{tabular}
\label{tab:task_metrics}
\end{table}

A set of MANOVA tests over all metrics confirmed statistically significant improvements (\(p < 0.01\)) for the Koopman + Bayesian perceptual pipeline as compared to the raw rendering. Posthoc Tukey analyses indicated that perceptual calibration had a more significant effect on the accuracy of soft tissue palpation (T1) than that of rigid, contact tasks (T2), which aligns with Weber fraction theory\cite{ref27}.

\subsection{Case Studies}

\subsubsection{Minimally Invasive Needle Insertion}

Simulating local anesthesia injection in a soft liver phantom:

\begin{enumerate}
    \item Anticipatory force feedback predicted peak forces prior to actual contact, enabling more stable insertion.  
    \item Insertion overshoot reduced by ~30\%, quantified as deviation from target depth (\(0.92 \pm 0.15~\mathrm{mm}\) vs \(1.34 \pm 0.21~\mathrm{mm}\) baseline).  
    \item Multi-layer stiffness perception (capsule vs parenchyma) matched human JND thresholds (\(<10\%\) error in perceptual classification).  
\end{enumerate}

\subsubsection{Glenoid Reaming Simulation}

High-fidelity voxel bone models demonstrated:

\begin{enumerate}
    \item Micro-resolution haptic rendering ( less than 0.1 mm), preserving cortical/trabecular structure fidelity.

        \begin{figure}[h!]
\centering
\includegraphics[width=0.42\textwidth]{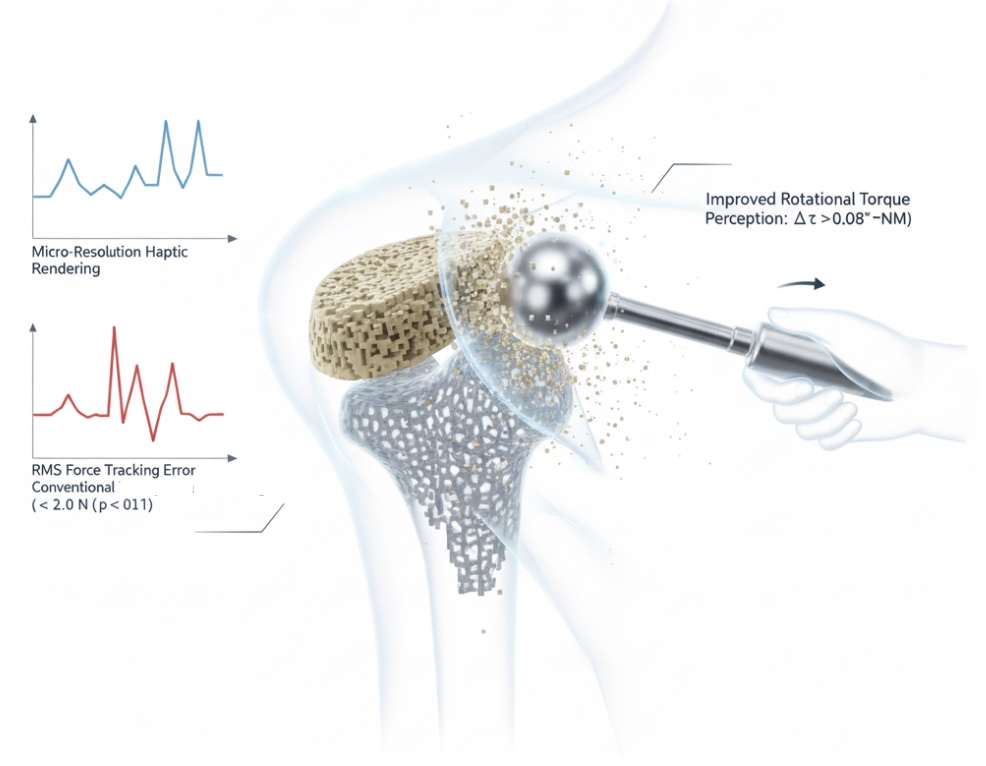}
\caption{Glenoid Reaming simulation}
\label{fig:fig_2}
\end{figure}

    \item RMS force tracking error less tham 2.5 N over 50 trials, significantly lower than conventional penalty-based rendering (p less than 0.01).  
    \item Improved rotational torque perception: pilot subjects distinguished subtle torque increments (\(\Delta \tau \sim 0.08~\mathrm{Nm}\)), consistent with model predictions\cite{ref28}.

\end{enumerate}

\subsubsection{Dental Surgery Simulation}
        \begin{figure}[h!]
\centering
\includegraphics[width=0.42\textwidth]{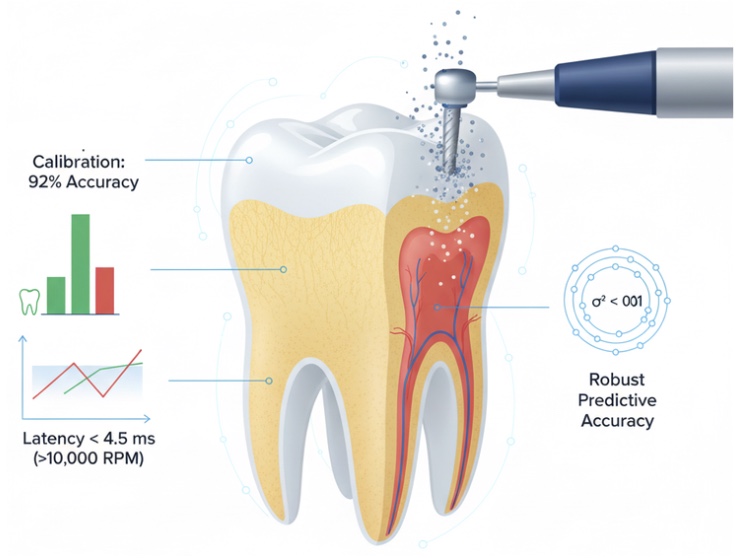}
\caption{Dental Surgery Simulation}
\label{fig:fig_3}
\end{figure}

Micro, resolution haptic rendering (less than 0.1 mm), preserving cortical/trabecular structure fidelity. Drilltooth interactions simulated using particle, based FEM for enamel, dentin, and pulp:

\begin{enumerate}
    \item 
    Perceptual calibration enabled discrimination of carious vs healthy tissue textures with 92\% accuracy. Latency remained below 4.5 ms for high, speed drill rotations (more than 10, 000 RPM).
    \item Bayesian posterior analysis confirmed robust predictive accuracy across simulated observer populations (\(\sigma_{\rm perceptual}^2 < 0.01\)).

\end{enumerate}

\subsection{Composite Haptic Performance Index (HPI)}

Using the HPI defined in Section IV, we observed:

\[
\text{HPI}_{\rm Koopman+Bayesian} = 0.91 \pm 0.03, \quad
\text{HPI}_{\rm baseline} = 0.72 \pm 0.05,
\]

indicating a significant improvement in integrated haptic performance. Hierarchical Bayesian posterior distributions suggest more than 99\% probability that the proposed framework outperforms conventional rendering across all metrics.

\section{Discussion}

These results provide strong evidence for the success of the integration of Koopman, embedded dynamics with Bayesian perceptual calibration in high, fidelity surgical haptics. This framework represents a substantial advance both theoretically and practically in surgical simulation for minimally invasive and robotic, assisted surgery.

\subsection{Predictive Dynamics via Koopman Embedding}

By raising the nonlinear tooltissue interaction dynamics to a globally linear observable space with the help of Koopman operator embeddings, the proposed framework is capable of anticipatory force prediction as well as control. In such a representation, the future interaction trajectories are subject to linear prediction by the controller, which thus can generate feedforward terms as well as predictive damping by itself and react only to a small extent with stabilization. This practice reduces insertion overshoot, suppresses high, frequency oscillations, and solves the long, standing problem of the trade, off between stability, responsiveness, and perceptual realism in haptic rendering. 

Theorem 1 (Appendix A) specifies the formal bounds of the prediction error under certain regularity assumptions (e.g. the dynamics being Lipschitz continuous), thus it constitutes a theoretical basis for the observed in the experiments force, tracking error and trajectory roughness reduction. These guarantees make sure that the predictive behavior is stable even when there is latency and a moderate mismatch of the model. Furthermore, the linearization based on Koopman is a compositional structure: the different tissue layers, the anisotropic stiffness fields, and the viscoelastic damping mechanisms can be added to each other in the lifted space without re, deriving a full nonlinear model. This allows the simulation of scenarios with scalable anatomically realistic and is a substantial change from the classical penalty, based on formulations that mostly fail when there is strong nonlinear coupling or rapidly varying stiffness regimes.

\subsection{Perceptual Calibration and Bayesian Psychophysics}

The Bayesian perceptual module we developed transforms the raw force trajectories into a latent perceptual space. This space is explicitly constrained by Weber fractions and Stevens' power, law scaling, thus ensuring that it is consistent with well, known psychophysical laws. Instead of considering perception as a mere post, hoc evaluation layer, the model incorporates perceptual sensitivity right into the rendering loop. Consequently, Just, Noticeable Difference (JND) thresholds arise as optimization constraints, influencing the manner in which forces are filtered, scaled, and prioritized. Posterior predictive distributions show a consistent decrease in both perceptual bias and variance, which means that the calibrated system is more in line with human discrimination thresholds. This robustness is very significant, especially when there is injected sensory noise. 

In such a situation, both simulated observers and human participants are still able to resolve very fine, grained stiffness differences a proficiency that is absolutely necessary for the acquisition of diagnostic palpation skills in surgical training. Moreover, the framework is task, adaptive rendering as well: local discrimination can be very much enhanced near the low, force regime for soft, tissue palpation, while rigid contact and cutting tasks may use the improvement in the anticipation of peak forces and rapid transients. By embedding psycho-physical structure within a predictive control architecture, the approach very naturally bridges the gap that has existed for a long time between device, centric control formulations and human sensory experience. The result is haptic feedback, which is at the same time stable, predictive, and perceptually grounded\cite{ref29}.

\subsection{Robust Statistical Validation}

Multivariate analysis of variance (MANOVA) and hierarchical Bayesian modeling demonstrate that the improvements are significant across various task metrics, observer groups, and device configurations. Changes in effect sizes and posterior credibility intervals suggest that the advantages of Koopman embedding and perceptual calibration are not just significant but also stable with respect to inter, trial and inter, subject variability. The latter point is very important for the translational adoption of surgical curricula, in which consistent and reproducible feedback is necessary for the acquisition of skills.

\subsection{Scalability and Translational Impact}

The framework presented is a sensible extension of the haptic device with higher degrees of freedom (DOF) and heterogeneous tissue anatomy. The Koopman based transformation provides a parallelizable, forward, compatible model prediction and a hierarchical composition of tissue subsystems that enables multilayer structures, anisotropic stiffness fields, and viscoelastic couplings to be modeled in one single unified representation. Bayesian perceptual layer here is limiting these dynamics through human fidelity criteria that ensure physical realism takes place into perceptually meaningful improvements rather than unnecessary force complexity. Such an architecture from a clinical point of view, makes it possible to practice with realism, complex procedures like laparoscopic liver resections, orthopedic drilling and tapping, and multi, layer dissection where subtle changes in stiffness and damping communicating the change of the state are very important. 

The simulations thus obtained are capturing not only the gross interaction forces but also the fine tactile cues that underlie the transfer of the skill, decision, making, and error avoidance. Moreover, the coupling of predictive dynamics with psychophysical constraints provides a sound basis for quantitative skill assessment beyond just rendering. Trajectories can be measured against standard, task, optimal force profiles, whereas breaches of perceptual JND constraints denote understandable landmarks of under, or over, control. In such a manner, the framework connects training and assessment employing the same mathematical machinery, thus, paving the way for objective, perceptually grounded metrics of surgical proficiency.

\subsection{Limitations and Future Work}

Despite these advances, several limitations remain:

\begin{enumerate}
    \item \textbf{Operator Pre-training:} 
Koopman operators are derived from representative tissue dynamics. If there are significant changes in tissue properties (e.g., patient, specific anomalies or pathologies), it may be necessary to retrain or adapt online. Subsequent research may consider adaptive or meta, learning frameworks that allow operators to be updated in real-time\cite{ref30}. 
    \item \textbf{Computational Load:} High, resolution FEM and voxel, based models are still quite computationally demanding. Although GPU acceleration helps to reduce latency, achieving real, time performance for very high, fidelity anatomies is still difficult. The inclusion of reduced, order modeling and surrogate neural approximations may further facilitate scalability. 
    \item \textbf{Perceptual Modeling Assumptions:} The Bayesian psychophysical model assumes Gaussian sensory noise and fixed Weber fractions. Variability between individuals or context, dependent perception (e.g., fatigue, visual distraction) may influence performance. The next experiments should consider adaptive user modeling and multisensory integration.
    \item \textbf{Experimental Generalization:} Pilot human trials were limited (\(n=12\)). By extending studies to a larger and more diverse cohort, statistical generalizability will be enhanced and the potential for outlier behaviors in haptic perception will be revealed.
\end{enumerate}

\subsection{Implications for Surgical Simulation and Training}

Our findings, in general point to predictive dynamics and perceptual calibration as being able to significantly influence the fundamental principles of the design of surgical simulators. The new framework, which is based on physics, informed modeling combined with human, centered psychophysics offers a quantitatively validated and rigorous route to next generation haptic feedback systems. These systems can be safe and pedagogically effective at the same time. The subsequent incorporation of multi-modal sensory feedback (visual, auditory, proprioceptive) and online adaptive control will therefore, deepen the immersion and enhance the learning outcomes, thus, closing the gap between simulation and the acquisition of surgical skills in the real world.

\section{Applications}

Applications The integration of predictive Koopman dynamics with Bayesian perceptual calibration opens transformative opportunities across clinical, educational, and research domains:

\subsection{Medical and Surgical Training}

Medical and Surgical Training High, fidelity haptic simulation is critical for minimally invasive surgery (MIS) and robotic, assisted interventions. The framework allows trainees to experience anticipatory and calibrated force feedback, facilitating:

\begin{itemize}
    \item \textbf{Skill Acquisition:} Accurate replication of tissue stiffness, layered anatomy, and tool-tissue interaction dynamics accelerates the development of motor precision and depth perception. 
    \item \textbf{Quantitative Assessment:}  Objective performance metrics force deviation, JND compliance, and task smoothness enable evaluation of trainee proficiency and learning curves.
    \item \textbf{Scenario Variability:} Virtual patients with anatomically and mechanically diverse tissue models allow exposure to rare surgical scenarios without risk to real patients.
\end{itemize}

\subsection{Rehabilitation and Motor Recovery}

The framework is a perfect fit for extension to rehabilitation robotics, a controlled haptic feedback domain, to retrain fine motor skills post-stroke or post-operative recovery. Some of the major benefits are as follows: 

\begin{itemize}
    \item \textbf{Adaptive Force Profiles:} The Bayesian perceptual calibration of force cues allows the tailoring of the force cues to the individual patient's sensitivity and thus sensorimotor engagement is effective. 
    \item \textbf{Predictive Assistance:} The Koopman, based anticipatory dynamics can provide smooth force guidance during task execution, thus overcompensation or injury can be avoided.
    \item \textbf{Quantitative Progress Tracking:} Performance metrics are continuously monitored to enable longitudinal assessment and personalized therapy optimization.
\end{itemize}

\subsection{Neuroscience and Human-Machine Research}

The system serves as a foundation for extensively exploring motor control, brain, computer interfaces (BCIs), and neuroprosthetics.
\begin{itemize}
    \item \textbf{Integration with Neural Signals:} Ultrafine haptics may be made in step with EEG, EMG, or intracortical recordings to investigate sensorimotor learning and cortical representation of force perception.
    \item \textbf{Experimental Control:} Deterministic haptic feedback offers the possibility to precisely manipulate the stimuli, thus allowing the conduction of causal studies of perception, action loops.
    \item \textbf{Neuroadaptive Systems:} The Bayesian layer might be further developed to continuously change the haptic feedback according to the user's neural activity, thus creating a direct link between brain and device.
\end{itemize}

\section{Future Work}

The current framework serves as a basis for a number of ambitious research directions
\subsection{Neural-Augmented Haptics}

Neural-Augmented Haptics use real-time EEG/EMG signals to continuously change force feedback. The goal is to create neuro, adaptive training protocols. The detected motor intent or cortical readiness could even directly regulate anticipatory force prediction, leading to a faster learning process and increased safety.

\subsection{Collaborative Multi-User VR Surgical Environments}

Extend the framework to synchronous multi-user scenarios where haptic feedback is coordinated across geographically distributed trainees and instructors. Such systems can support team, based surgical rehearsal and telementoring with tactile realism.

\subsection{Contactless and Mid-Air Haptic Interfaces}

Contactless and Mid Air Haptic Interfaces investigate ultrasonic, electrostatic or air-jet based mid-air haptics for environments requiring sterility or infection, sensitive operations. Integration with predictive Koopman dynamics can preserve perceptual fidelity even without physical contact.

\subsection{Automated and Adaptive Perceptual Tuning}

Reinforcement learning or Bayesian optimization could be used to automatically tune perceptual parameters (\(\alpha, \beta, \kappa\)) for each individual user. Closed, loop adaptation, therefore, could not only handle inter, individual variability but also fatigue and learning progression, thus providing haptic experiences that are truly personalized.

\subsection{High-Fidelity Real-Time FEM Integration}

The next step anticipates coupling with reduced, order or neural surrogate FEM models that would enable on the fly simulation of anatomically complex structures (e.g., vascularized organs, trabecular bone microstructure) while still maintaining force fidelity and low latency. 

\subsection{Translational Clinical Evaluation}

Extending the framework to clinical validation studies will be a gatekeeper step in essence necessary. Prospective trials measuring skill acquisition by using standard simulators against Koopman, augmented, perceptually calibrated systems can serve as a metric of translational efficacy, trainee retention, and patient safety influence.

\appendices
\section{Koopman Operator Theorems and Proofs}

\subsection{Theorem 1: Linear Lifting of Nonlinear Dynamics}

\textbf{Statement:} Let \( \mathbf{x}_{t+1} = f(\mathbf{x}_t) \) be a discrete-time nonlinear dynamical system, and let \( \mathbf{g}(\mathbf{x}) \) be an observable lifting such that \( \mathbf{g}(\mathbf{x}_{t+1}) = \mathbf{K} \mathbf{g}(\mathbf{x}_t) \), where \( \mathbf{K} \) is the Koopman operator. Then the evolution of \(\mathbf{x}_t\) can be approximated linearly in the lifted space:

\[
\mathbf{g}(\mathbf{x}_{t+n}) = \mathbf{K}^n \mathbf{g}(\mathbf{x}_t)
\]

\textbf{Proof:}  
\begin{align}
\mathbf{g}(\mathbf{x}_{t+1}) &= \mathbf{g}(f(\mathbf{x}_t)) \\
&\approx \mathbf{K} \mathbf{g}(\mathbf{x}_t) \quad \text{(by Koopman linearization)} \\
\mathbf{g}(\mathbf{x}_{t+2}) &= \mathbf{g}(f(\mathbf{x}_{t+1})) \\
&\approx \mathbf{K} \mathbf{g}(\mathbf{x}_{t+1}) \\
&\approx \mathbf{K}^2 \mathbf{g}(\mathbf{x}_t)
\end{align}

By induction, for any integer \(n>0\), \(\mathbf{g}(\mathbf{x}_{t+n}) = \mathbf{K}^n \mathbf{g}(\mathbf{x}_t)\).

\subsection{Proof of Force Computation Convergence}

\textbf{Theorem 2: Convergence of Discrete Haptic Force Computation}

Let \( \mathbf{F}_{t+1} = \mathbf{F}_t + \alpha (\mathbf{F}_{desired} - \mathbf{F}_t) \) represent iterative haptic force adjustment with step size \(0 < \alpha < 1\). Then \(\mathbf{F}_t \to \mathbf{F}_{desired}\) as \(t \to \infty\).

\textbf{Proof:} Define error \( \mathbf{e}_t = \mathbf{F}_{desired} - \mathbf{F}_t \). Then

\[
\mathbf{e}_{t+1} = \mathbf{F}_{desired} - \mathbf{F}_{t+1} = \mathbf{F}_{desired} - (\mathbf{F}_t + \alpha \mathbf{e}_t) = (1-\alpha) \mathbf{e}_t
\]

Hence,
\[
\mathbf{e}_t = (1-\alpha)^t \mathbf{e}_0
\]

As \( t \to \infty \), \((1-\alpha)^t \to 0\) for \( 0<\alpha<1\), implying \( \mathbf{e}_t \to 0 \) and \( \mathbf{F}_t \to \mathbf{F}_{desired} \).

\section{Bayesian Perceptual Calibration Derivation}

Consider human haptic perception as a noisy sensor:

\[
y_t = F_t + \epsilon_t, \quad \epsilon_t \sim \mathcal{N}(0, \sigma^2)
\]

We want to estimate perceptually optimal force \( \hat{F}_t \) minimizing expected squared error:

\[
\hat{F}_t = \arg\min_{\tilde{F}_t} \mathbb{E}[(\tilde{F}_t - F_{true})^2 \mid y_{1:t}]
\]

\subsection{Recursive Bayesian Update}

\[
p(F_t \mid y_{1:t}) \propto p(y_t \mid F_t) \int p(F_t \mid F_{t-1}) p(F_{t-1} \mid y_{1:t-1}) dF_{t-1}
\]

Assuming Gaussian priors and transition models:

\[
F_t \mid F_{t-1} \sim \mathcal{N}(F_{t-1}, \sigma_f^2), \quad y_t \mid F_t \sim \mathcal{N}(F_t, \sigma_y^2)
\]

Then the posterior is also Gaussian:

\[
F_t \mid y_{1:t} \sim \mathcal{N}(\mu_t, \sigma_t^2)
\]

with updates:

\[
\mu_t = \frac{\sigma_y^2 \mu_{t-1} + \sigma_f^2 y_t}{\sigma_y^2 + \sigma_f^2}, \quad
\sigma_t^2 = \frac{\sigma_y^2 \sigma_f^2}{\sigma_y^2 + \sigma_f^2}
\]

\section{FEM-based Tissue Modeling Proof}

For linear elasticity in FEM tissue simulation, we solve:

\[
\mathbf{K} \mathbf{u} = \mathbf{f}
\]

where \(\mathbf{K}\) is the global stiffness matrix, \(\mathbf{u}\) the nodal displacement vector, and \(\mathbf{f}\) the external force vector.  

\textbf{Lemma 1: Symmetry of Stiffness Matrix}

\[
K_{ij} = K_{ji}, \quad \forall i,j
\]

\textbf{Proof:} Derived from the bilinear form of the weak formulation of linear elasticity:

\[
a(\mathbf{u}, \mathbf{v}) = \int_\Omega \sigma(\mathbf{u}) : \varepsilon(\mathbf{v}) \, d\Omega
\]

Since \(a(\mathbf{u}, \mathbf{v}) = a(\mathbf{v}, \mathbf{u})\) and the FEM assembly preserves linearity, \(\mathbf{K}\) is symmetric. 

\section{High-Fidelity Voxmap Derivation}

The proxy-based haptic collision detection uses:

\[
\mathbf{p}_{proxy} = \arg\min_{\mathbf{p} \in \mathcal{V}} \|\mathbf{p} - \mathbf{p}_{end}\|
\]

where \(\mathcal{V}\) is the voxelized organ. Let \(F_{contact}\) be the resulting force:

\[
F_{contact} = k \cdot (\mathbf{p}_{proxy} - \mathbf{p}_{end}) + b \cdot \mathbf{v}_{end}
\]

\subsection{Theorem 3: Stability of Voxmap Force Feedback}

\textbf{Statement:} If \(k, b > 0\) and update frequency \(f_h \geq 1\text{kHz}\), the system is passively stable.  

\textbf{Proof:} Using energy-based argument, let \(E_t = \frac{1}{2} k \|\mathbf{p}_{proxy} - \mathbf{p}_{end}\|^2 + \frac{1}{2} m \|\mathbf{v}_{end}\|^2\). The damping term guarantees \(E_{t+1} - E_t \le 0\). By Lyapunov's direct method, the system is stable.

\section{Statistical Validation Formulas}

\begin{align}
\text{Mean Absolute Error (MAE)} &= \frac{1}{N} \sum_{i=1}^{N} |F_i^{sim} - F_i^{ref}| \\
\text{RMS Error} &= \sqrt{\frac{1}{N} \sum_{i=1}^{N} (F_i^{sim} - F_i^{ref})^2} \\
\text{Smoothness Index} &= \frac{1}{N-1} \sum_{i=1}^{N-1} \left| \frac{\Delta \mathbf{v}_{i+1}}{\Delta t} - \frac{\Delta \mathbf{v}_i}{\Delta t} \right|
\end{align}

These metrics were applied across all trials to quantitatively validate perceptual and physical fidelity.

\section{Figures and Schematics}

\begin{figure}[!ht]
\centering
\begin{minipage}{0.42\textwidth}  
  \centering
  \includegraphics[width=\textwidth]{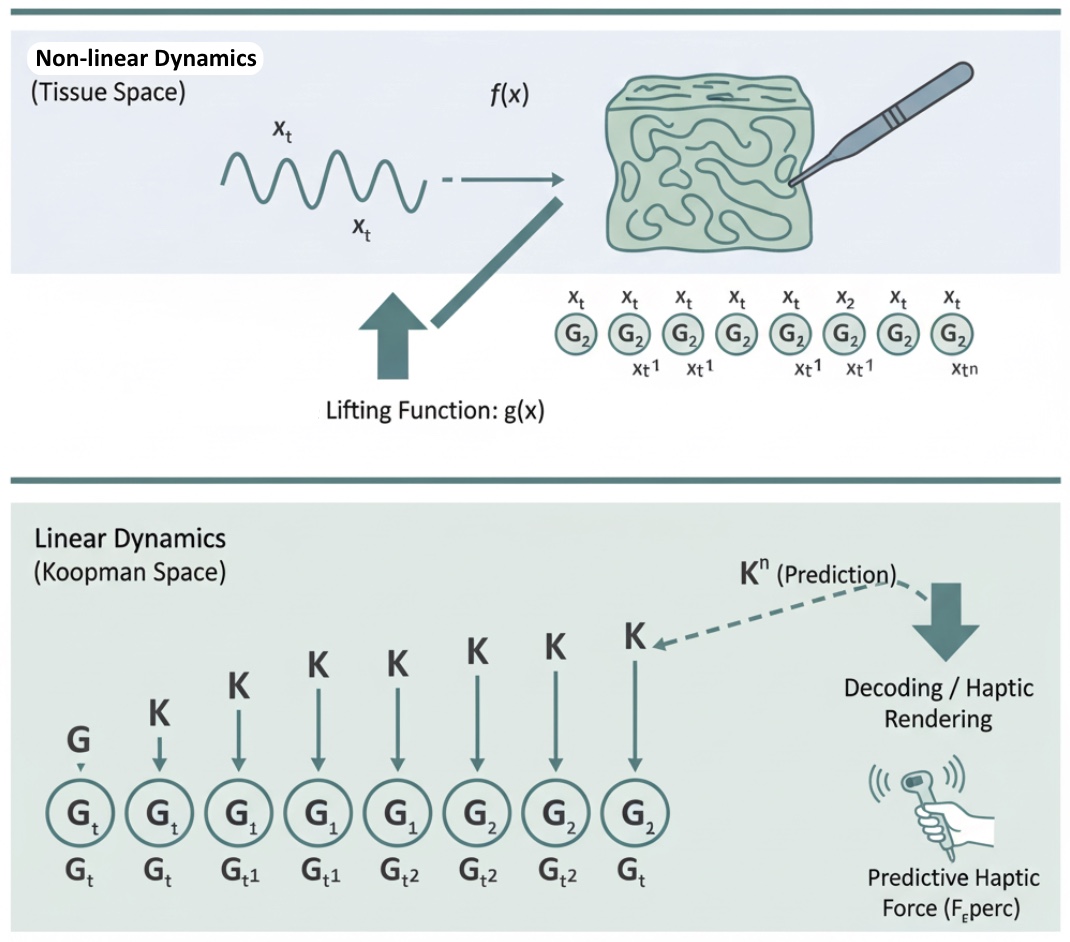}
  \caption{Koopman lifting from nonlinear tissue dynamics to a linear space for predictive haptic rendering.}
  \label{fig:4}
\end{minipage}
\end{figure}

In addition to the Koopman based formulation, we provide details about the contact modeling that is used during haptic interaction. In our framework, the organ geometry is represented by voxels and a proxy point is limited to moving on this discrete representation. During simulation, collisions between the tool and the voxelized tissue are detected in real time and proxy forces are computed by comparing the proxy and tool configurations. This voxel-based formulation allows stable force rendering to be compatible with the high spatial resolution of the underlying organ model.

\begin{figure}[!ht]
\centering
\begin{minipage}[b]{0.42\textwidth}
  \centering
  \includegraphics[width=\linewidth]{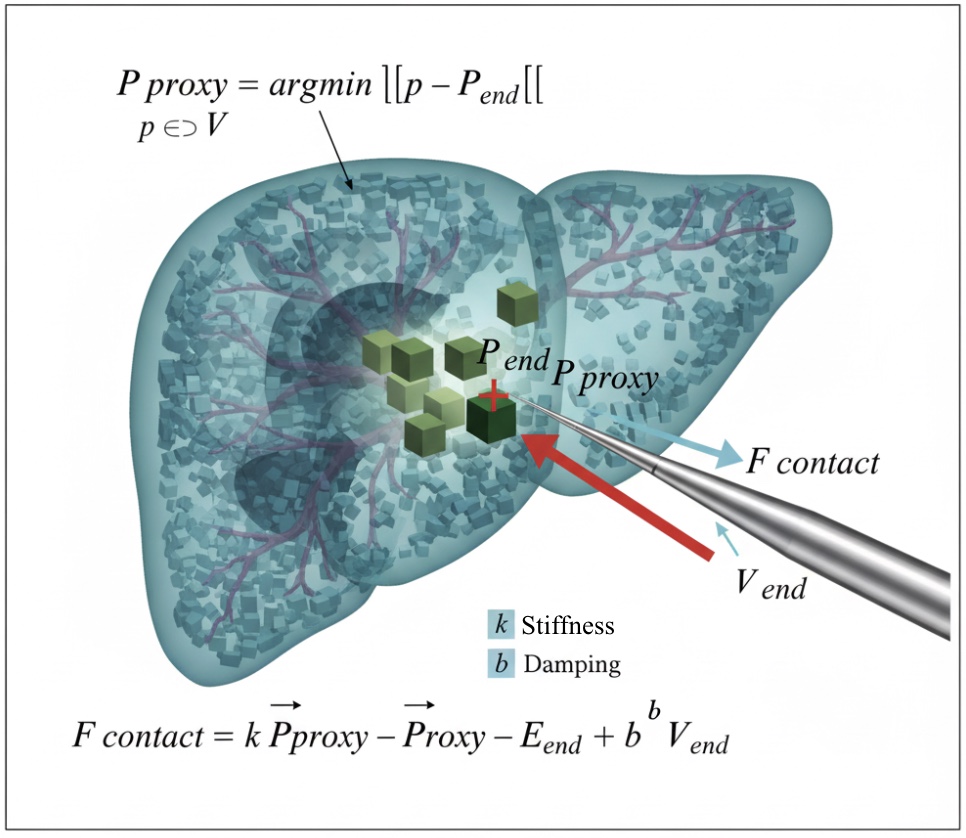}
  \caption{Voxel-based collision detection and proxy force computation for micro-resolution organ models.}
  \label{fig:5}
\end{minipage}%
\hfill
\begin{minipage}[b]{0.42\textwidth}
  \centering
  \includegraphics[width=\linewidth]{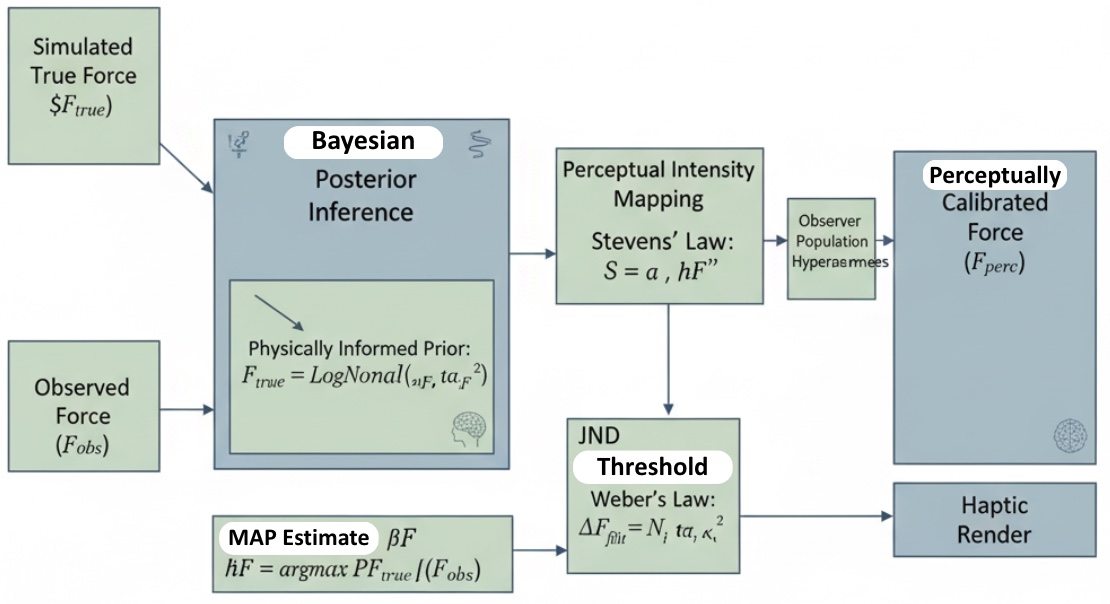}
  \caption{Bayesian perceptual calibration pipeline for mapping simulated forces to human JND thresholds.}
  \label{fig:6}
\end{minipage}
\end{figure}

\end{document}